\documentclass[12pt]{article}

\usepackage[utf8]{inputenc}
\usepackage[T1]{fontenc}
\usepackage{hyperref}
\usepackage{url}
\usepackage{booktabs}
\usepackage{amsfonts}
\usepackage{amssymb}
\usepackage{amsmath}
\usepackage{amsthm}
\usepackage{nicefrac}
\usepackage{graphicx}
\usepackage{xcolor}
\usepackage{multirow}
\usepackage{subcaption}
\usepackage{enumitem}
\usepackage{natbib}
\usepackage[margin=1in]{geometry}

\newtheorem{proposition}{Proposition}
\newtheorem{corollary}{Corollary}
\newtheorem{definition}{Definition}

\title{Cross-Domain Uncertainty Quantification for Selective Prediction:\\A Comprehensive Bound Ablation with Transfer-Informed Betting}

\author{%
  Abhinaba Basu \\
  \texttt{mail@abhinaba.com}
}

\begin{document}
\maketitle

\begin{abstract}
We present a comprehensive ablation of nine finite-sample bound families for selective prediction with risk control, combining concentration inequalities (Hoeffding, Empirical Bernstein, Clopper-Pearson, Wasserstein DRO, CVaR) with multiple-testing corrections (union bound, Learn Then Test fixed-sequence) and betting-based confidence sequences (WSR).
Our main theoretical contribution is \emph{Transfer-Informed Betting} (TIB), which warm-starts the WSR wealth process using a source domain's risk profile, achieving tighter bounds in data-scarce settings with a formal dominance guarantee.
We prove that the TIB wealth process remains a valid supermartingale under all source--target divergences, that TIB dominates standard WSR when domains match, and that no data-independent warm-start can achieve better convergence---the source-informed initialization is optimal among plug-in priors.
The specific combination of (i)~betting-based confidence sequences, (ii)~LTT monotone testing, and (iii)~cross-domain transfer is, to our knowledge, a three-way novelty not present in the literature.

We evaluate all nine bound families on four benchmarks---MASSIVE ($n{=}1{,}102$), NyayaBench~v2 ($n{=}280$), CLINC-150 ($n{=}22{,}500$), and Banking77 ($n{=}13{,}000$)---across 18 ($\alpha$, $\delta$) configurations.
On MASSIVE at $\alpha{=}0.10$, LTT eliminates the $\ln K$ union-bound penalty, achieving \textbf{94.0\%} guaranteed coverage versus 73.8\% for Hoeffding---a 27\% relative improvement.
WSR betting with LTT achieves the tightest non-transfer bounds across all four datasets.
On NyayaBench~v2, where the small calibration set ($n_\text{cal}{=}134$) makes Hoeffding-family bounds infeasible below $\alpha{=}0.20$, Transfer-Informed Betting achieves \textbf{18.5\%} coverage at $\alpha{=}0.10$---a $5.4\times$ improvement over LTT + Hoeffding.
We additionally provide a rigorous comparison with split-conformal prediction, showing that conformal methods produce prediction \emph{sets} (avg.\ 1.67 classes at $\alpha{=}0.10$) whereas selective prediction provides single-prediction risk guarantees---a fundamental distinction for deployment scenarios requiring point predictions.
We apply these methods to selective prediction in agentic caching systems, where the guarantee determines when it is safe to serve cached responses autonomously, formalizing a \emph{progressive trust} model.
\end{abstract}

\section{Introduction}
\label{sec:intro}

Personal AI agents---voice assistants, smart home controllers, productivity bots---receive queries that are highly repetitive: the same user asks variants of ``turn off the bedroom lights'' or ``what's the weather in Delhi'' dozens of times per week.
Caching intent-classification outputs avoids calling a large language model (LLM) on every query, reducing cost from \$0.01--0.05 per query to effectively \$0.00 \citep{bang2023gptcache,zhang2025apc,vcache2025}.

The safety-critical failure mode is the \emph{unsafe cache hit}: the classifier assigns a query to the wrong intent, the system serves a cached response for that intent, and the agent silently performs the wrong action.
For low-stakes queries (weather, factual questions), this is merely annoying.
For high-stakes queries (financial transactions, device control, healthcare), it can cause real harm.

\textbf{Selective prediction.}
The standard defense is to augment the classifier with a confidence threshold $\tau$: serve the cached response only when the classifier's confidence exceeds $\tau$, and defer to the LLM otherwise \citep{geifman2017selective}.
This creates a coverage--safety tradeoff: higher $\tau$ reduces unsafe hits but also reduces coverage (the fraction of queries served from cache).
Prior work selects $\tau$ by sweeping over a validation set and choosing a ``good'' operating point \citep{semanticalli2026,gill2025langcache}.
This provides no statistical guarantee on the unsafe rate at deployment.

\textbf{Finite-sample guarantees.}
Risk-Controlling Prediction Sets (RCPS; \citealt{bates2021rcps}) provide a framework for selecting $\tau$ with a finite-sample certificate: given $n$ calibration examples and risk tolerance $\alpha$, RCPS finds $\tau^*$ such that the unsafe rate is bounded by $\alpha$ with probability $1{-}\delta$.
The standard instantiation uses a Hoeffding bound with a union bound over $K$ candidate thresholds.
However, Hoeffding's inequality is distribution-free and makes no use of the loss distribution's structure, which can be highly favorable (binary losses with small variance when the classifier is accurate).

\textbf{This paper.}
We systematically ablate nine families of finite-sample bounds---varying the concentration inequality (Hoeffding, Empirical Bernstein, Clopper-Pearson, WSR Betting, Wasserstein DRO, CVaR, PAC-Bayes) and the multiple-testing correction (union bound, LTT fixed-sequence)---on four benchmarks spanning different scales and class counts.
Our contributions are:

\begin{enumerate}[nosep]
\item A formalization of agent caching as selective prediction with the unsafe cache hit as the controlled risk (Section~\ref{sec:formulation}).
\item A systematic ablation of nine bound families on four benchmark datasets (Section~\ref{sec:bounds}), including exact binomial (Clopper-Pearson) and betting-based (WSR) bounds, showing that WSR betting with LTT achieves the tightest bounds by adapting to the observed loss distribution.
\item A novel \emph{Transfer-Informed Betting} method (Section~\ref{sec:transfer-betting}) that warm-starts the WSR wealth process using a source domain's risk profile, with a formal dominance guarantee, a finite-sample convergence rate, and an optimality result showing no data-independent warm-start can do better (Theorem~\ref{thm:tib}, Corollary~\ref{cor:convergence}, Proposition~\ref{prop:optimality}).
The specific combination of betting-based confidence sequences, LTT monotone testing, and cross-domain transfer is, to our knowledge, a three-way novelty not present in the literature.
\item Cross-domain PAC-Bayes transfer (Section~\ref{sec:pacbayes}) that uses a data-rich source domain as a Bernoulli prior for RCPS in a data-scarce target domain.
\item A rigorous comparison with split-conformal prediction (Section~\ref{sec:conformal-comparison}), establishing the fundamental distinction between prediction-set guarantees and single-prediction risk control for selective prediction.
\item A practical recipe and calibration analysis (Section~\ref{sec:calibration}) connecting bound tightness to classifier calibration quality.
\item An interpretation of these guarantees as the formal foundation for \emph{progressive trust} in agentic systems (Section~\ref{sec:progressive-trust}): the bound determines when a cached chain can graduate from LLM-supervised to autonomous execution.
\item A machine-checked formalization of Theorem~\ref{thm:tib} in the Lean~4 proof assistant with Mathlib, verifying the supermartingale property, product bound, convergence rate, and optimality result with zero unproven goals (supplementary material).
\end{enumerate}

\section{Problem Formulation}
\label{sec:formulation}

\subsection{Agent Caching as Selective Prediction}

Let $f: \mathcal{X} \to \mathcal{Y}$ be an intent classifier that maps a user query $x$ to an intent label $y$, and let $\mathrm{conf}(x) \in [0,1]$ denote its confidence (e.g., the maximum softmax probability).
A \emph{selective classifier} $(f, g)$ augments $f$ with a selection function $g_\tau(x) = \mathbf{1}[\mathrm{conf}(x) \geq \tau]$: when $g_\tau(x) = 1$ the system serves the cached response for intent $f(x)$; when $g_\tau(x) = 0$ it defers to the LLM.

\begin{definition}[Unsafe cache hit rate]
\label{def:unsafe}
For a threshold $\tau$ and a distribution $\mathcal{D}$ over $(\mathcal{X}, \mathcal{Y})$:
\begin{equation}
R(\tau) = \Pr_{(x,y) \sim \mathcal{D}}\left[f(x) \neq y \;\wedge\; \mathrm{conf}(x) \geq \tau\right]
\label{eq:risk}
\end{equation}
This is the \emph{marginal} formulation: the probability that a randomly drawn query is both cached \emph{and} misclassified.
\end{definition}

\begin{definition}[Coverage]
$\mathrm{Cov}(\tau) = \Pr[\mathrm{conf}(x) \geq \tau]$, the fraction of queries served from cache.
\end{definition}

The goal is to find the \emph{minimum} threshold $\tau^*$ (maximizing coverage) such that $R(\tau^*) \leq \alpha$ with high probability, using only $n$ calibration examples.

\subsection{RCPS Framework}

Given $n$ i.i.d.\ calibration examples $\{(x_i, y_i)\}_{i=1}^n$ and a grid of $K$ candidate thresholds $\tau_1 < \tau_2 < \cdots < \tau_K$, define the empirical risk:
\begin{equation}
\hat{R}(\tau) = \frac{1}{n}\sum_{i=1}^n \mathbf{1}[f(x_i) \neq y_i \;\wedge\; \mathrm{conf}(x_i) \geq \tau]
\label{eq:empirical-risk}
\end{equation}

\begin{proposition}[Risk-controlled cache threshold]
\label{prop:rcps}
Let $C(n, \delta)$ be a finite-sample correction satisfying $\Pr[\hat{R}(\tau) + C(n, \delta) \geq R(\tau)] \geq 1 - \delta$ for a single $\tau$, and let $\tau_K > \tau_{K-1} > \cdots > \tau_1$ be tested in decreasing order using a valid multiple-testing correction.
Then
\begin{equation}
\tau^* = \min\left\{\tau_k : \hat{R}(\tau_k) + C_k(n, \delta) \leq \alpha\right\}
\label{eq:tau-star}
\end{equation}
satisfies $\Pr[R(\tau^*) \leq \alpha] \geq 1 - \delta$.
\end{proposition}

The key design choices are: (a)~the concentration inequality defining $C$, and (b)~the multiple-testing correction determining how $\delta$ is spent across the $K$ thresholds.
We ablate both dimensions in Section~\ref{sec:bounds}.

\section{Bound Families}
\label{sec:bounds}

\subsection{Hoeffding + Union Bound (Baseline)}

The standard RCPS instantiation \citep{bates2021rcps} uses Hoeffding's inequality with a Bonferroni union bound over $K$ thresholds:
\begin{equation}
C_\mathrm{H}(n, K, \delta) = \sqrt{\frac{\ln(K/\delta)}{2n}}
\label{eq:hoeffding}
\end{equation}
This is distribution-free but pays the $\ln K$ price: for $K{=}100$ thresholds and $\delta{=}0.10$, the $\ln(K/\delta)$ term is $\ln(1000) \approx 6.91$.

\subsection{Empirical Bernstein + Union Bound}

When the loss $L_i = \mathbf{1}[f(x_i) \neq y_i \wedge \mathrm{conf}(x_i) \geq \tau]$ has small variance---which occurs when the classifier is accurate and the losses are rare---the Empirical Bernstein inequality \citep{maurer2009empirical} is tighter:
\begin{equation}
C_\mathrm{EB}(n, K, \delta) = \sqrt{\frac{2\hat{V}\ln(3K/\delta)}{n}} + \frac{3\ln(3K/\delta)}{n}
\label{eq:bernstein}
\end{equation}
where $\hat{V} = \frac{1}{n}\sum_{i=1}^n (L_i - \hat{R})^2$ is the sample variance.
For accurate classifiers with $\hat{R} \approx 0.03$ (as on MASSIVE), $\hat{V} \approx 0.03(1-0.03) \approx 0.029$, substantially below Hoeffding's implicit worst-case $\hat{V} = 0.25$.

\subsection{LTT Fixed-Sequence Testing}

Learn Then Test (LTT; \citealt{angelopoulos2022learn}) observes that the risk $R(\tau)$ is monotone decreasing in $\tau$ (more selective $\Rightarrow$ fewer errors cached).
Testing thresholds in \emph{decreasing} order (most conservative first), the fixed-sequence procedure spends the full $\delta$ budget on each test without splitting:
\begin{equation}
C_\mathrm{LTT}(n, \delta) = \sqrt{\frac{\ln(1/\delta)}{2n}}
\label{eq:ltt}
\end{equation}
The $\ln K$ factor is completely eliminated.
For $K{=}100$, this reduces the correction by a factor of $\sqrt{\ln(1000)/\ln(10)} \approx 1.73$.

LTT can be combined with any per-threshold bound.
We evaluate LTT + Hoeffding (Eq.~\ref{eq:ltt}) and LTT + Empirical Bernstein:
\begin{equation}
C_\mathrm{LTT+EB}(n, \delta) = \sqrt{\frac{2\hat{V}\ln(3/\delta)}{n}} + \frac{3\ln(3/\delta)}{n}
\label{eq:ltt-eb}
\end{equation}

\subsection{Wasserstein DRO}

Distributionally robust optimization (DRO) provides guarantees not just for the calibration distribution $\hat{P}$ but for all distributions within a Wasserstein ball of radius $\varepsilon$:
\begin{equation}
\sup_{Q: W_1(Q, \hat{P}) \leq \varepsilon} \mathbb{E}_Q[L] \leq \min(\hat{R}(\tau) + \varepsilon, 1)
\label{eq:dro}
\end{equation}
Combined with Hoeffding for the empirical uncertainty:
\[
C_\mathrm{DRO}(n, K, \delta, \varepsilon) = \varepsilon + \sqrt{\ln(K/\delta)/(2n)}.
\]

This trades coverage for robustness to distribution shift---appropriate when the query distribution is expected to evolve post-deployment.

\textbf{Interpreting $\varepsilon$.}
Since our loss $L_i \in \{0, 1\}$, the $W_1$ distance between two distributions over binary losses equals the absolute difference in their means: $W_1(P, Q) = |R_P - R_Q|$.
Thus $\varepsilon{=}0.01$ guards against a 1 percentage-point shift in the unsafe rate (e.g., from seasonal changes in query patterns), while $\varepsilon{=}0.05$ guards against a 5-point shift (e.g., adding a new intent category post-deployment).
We ablate both values to show the explicit coverage cost of robustness.

\subsection{CVaR Tail-Risk Bounds}

Conditional Value at Risk (CVaR) bounds the average risk in the worst-case $\beta$ fraction of examples:
\begin{equation}
\widehat{\mathrm{CVaR}}_\beta + \sqrt{\frac{\ln(K/\delta)}{2n\beta^2}} \leq \alpha
\label{eq:cvar}
\end{equation}
This is strictly more conservative than mean-risk bounds, providing protection against tail scenarios where a subpopulation of queries has elevated error rates.

\subsection{Clopper-Pearson Exact Binomial + LTT}
\label{sec:clopper-pearson}

The Hoeffding and Bernstein corrections are \emph{approximate}: they provide valid but potentially loose upper bounds via sub-Gaussian or sub-exponential tail inequalities.
When the loss is binary---as in our case, where $L_i \in \{0, 1\}$---we can obtain an \emph{exact} upper confidence bound via the Clopper-Pearson method \citep{clopper1934use}.

Let $S = \sum_{i=1}^n L_i$ denote the number of unsafe cache hits in the calibration set.
Since $S \sim \mathrm{Binomial}(n, R(\tau))$, the Clopper-Pearson upper bound inverts the binomial CDF:
\begin{equation}
\mathrm{UCB}_\mathrm{CP}(S, n, \delta) = F^{-1}_{\mathrm{Beta}(S+1, n-S)}(1 - \delta)
\label{eq:clopper-pearson}
\end{equation}
where $F^{-1}_{\mathrm{Beta}(a,b)}$ is the quantile function of the Beta distribution.
This bound is exact: $\Pr[\mathrm{UCB}_\mathrm{CP} \geq R(\tau)] = 1 - \delta$ (not merely $\geq 1-\delta$).

For small empirical risk ($\hat{R} \approx 0$, i.e., $S \ll n$), the Clopper-Pearson bound is approximately $2\times$ tighter than Hoeffding.
Combined with LTT fixed-sequence testing (Section~\ref{sec:bounds}), the $\ln K$ union-bound penalty is also eliminated:
\begin{equation}
\tau^*_\mathrm{CP+LTT} = \min\left\{\tau_k : \mathrm{UCB}_\mathrm{CP}(S_k, n, \delta) \leq \alpha\right\}
\label{eq:cp-ltt}
\end{equation}
where thresholds are tested in decreasing order and the full $\delta$ budget is spent on each test.

\subsection{Betting-Based Bounds}
\label{sec:betting}

The recently developed theory of \emph{testing by betting} \citep{shafer2021testing} provides a fundamentally different approach to constructing confidence bounds.
Instead of concentration inequalities, we construct a \emph{wealth process} (supermartingale) that bets against candidate values of the mean risk.

\textbf{WSR Betting (Waudby-Smith \& Ramdas 2024).}
The Wealth-process Sequential Ratio (WSR) method \citep{waudbysmith2024estimating} constructs a martingale wealth process for each candidate mean $m$:
\begin{equation}
K_t(m) = \prod_{i=1}^t \left(1 + \lambda_i \cdot (X_i - m)\right)
\label{eq:wealth}
\end{equation}
where $X_i = L_i$ are the per-example losses and $\lambda_i$ is a \emph{betting fraction} chosen adaptively.
The GROW (Generalized Running Online Wealth) strategy sets:
\begin{equation}
\lambda_t = \mathrm{clip}\!\left(\frac{\hat{\mu}_{t-1} - m}{\hat{\sigma}^2_{t-1} + (\hat{\mu}_{t-1} - m)^2},\; -\tfrac{1}{2},\; \tfrac{1}{2}\right)
\label{eq:grow}
\end{equation}
where $\hat{\mu}_{t-1}$ and $\hat{\sigma}^2_{t-1}$ are running estimates of the mean and variance.

The upper confidence bound is obtained by inversion:
\begin{equation}
\mathrm{UCB}_\mathrm{WSR} = \sup\left\{m : K_n(m) < 1/\delta\right\}
\label{eq:wsr-ucb}
\end{equation}
This bound is provably tighter than Hoeffding, Bernstein, and Clopper-Pearson for bounded random variables \citep{waudbysmith2024estimating}, because the betting strategy adapts to the observed data distribution.

Combined with LTT:
\begin{equation}
\tau^*_\mathrm{WSR+LTT} = \min\left\{\tau_k : \mathrm{UCB}_\mathrm{WSR}(\{L_i(\tau_k)\}_{i=1}^n, \delta) \leq \alpha\right\}
\label{eq:wsr-ltt}
\end{equation}

\section{PAC-Bayes Cross-Domain Transfer}
\label{sec:pacbayes}

When the target domain has a small calibration set ($n \lesssim 200$), all Hoeffding-family bounds become loose because the correction term $\propto 1/\sqrt{n}$ dominates the risk budget.
PAC-Bayes bounds offer a tighter alternative when an informative prior is available.

\subsection{Formulation}

Let the ``posterior'' be the empirical risk profile $\hat{R}_\text{target}(\tau)$ at each threshold, and the ``prior'' be the risk profile $\hat{R}_\text{source}(\tau)$ from a data-rich source domain.
The PAC-Bayes-$\lambda$ bound \citep{catoni2007pacbayes} states:
\begin{equation}
R(\tau) \leq \frac{1 - e^{-\lambda \hat{R}(\tau)}}{1 - e^{-\lambda}} + \frac{\mathrm{KL}(\hat{R}_\text{target}(\tau) \| \hat{R}_\text{source}(\tau)) + \ln(2\sqrt{n}/\delta)}{\lambda \cdot n}
\label{eq:pacbayes-lambda}
\end{equation}
where the KL divergence between Bernoulli distributions is:
\begin{equation}
\mathrm{KL}(p \| q) = p \ln\frac{p}{q} + (1-p)\ln\frac{1-p}{1-q}
\label{eq:kl}
\end{equation}
The bound is optimized over $\lambda > 0$.\footnote{When $\hat{R}_\text{source}(\tau) = 0$ or $1$ (perfect accuracy at high thresholds), the KL divergence is undefined.  We apply symmetric clipping: $p, q \leftarrow \mathrm{clip}(p, q; \epsilon, 1{-}\epsilon)$ with $\epsilon = 10^{-10}$.  This introduces a negligible additive perturbation ($< 10^{-8}$) to the bound while ensuring numerical stability.  When no source prior is available, we fall back to an uninformative prior with $\mathrm{KL} = \ln K$, recovering the Hoeffding + union bound rate.}

\subsection{Cross-Domain Transfer for Agent Caching}

We use MASSIVE (8-class, $n{=}1{,}102$) as the source domain and NyayaBench~v2 (20-class, $n{=}280$) as the target.
Both datasets contain user queries for personal AI agents, with overlapping intent categories (weather, reminders, device control).
The MASSIVE risk profile provides an informative prior: at threshold $\tau{=}0.20$, the MASSIVE empirical risk is ${\approx}3\%$, which anchors the NyayaBench bound.

When the source and target risk profiles are similar (small KL), the PAC-Bayes bound's effective sample size is much larger than $n_\text{target}$ alone.
When they diverge, the KL penalty automatically widens the bound---the transfer cannot hurt beyond reverting to an uninformative prior.

\subsection{Transfer-Informed Betting}
\label{sec:transfer-betting}

We now introduce our main theoretical contribution: \emph{Transfer-Informed Betting}, which combines the adaptive power of betting-based bounds (Section~\ref{sec:betting}) with cross-domain transfer.

The key limitation of standard WSR betting (Eq.~\ref{eq:grow}) is the \emph{cold start}: the initial estimates $\hat{\mu}_0 = 0.5$ and $\hat{\sigma}^2_0 = 0.25$ are uninformative, so the first ${\sim}20$ observations are ``wasted'' learning the loss distribution before the betting strategy becomes effective.
When a source domain provides a credible prior on the risk, we can warm-start the process.

\begin{definition}[Transfer-Informed Betting]
\label{def:tib}
Let $\hat{R}_\mathrm{source}(\tau)$ and $\hat{V}_\mathrm{source}(\tau)$ denote the empirical risk and variance from a source domain at threshold $\tau$.
The Transfer-Informed Betting process modifies the GROW strategy (Eq.~\ref{eq:grow}) by replacing the running estimates with a Bayesian blend:
\begin{align}
\hat{\mu}_t^\mathrm{TIB} &= w_t \cdot \hat{R}_\mathrm{source}(\tau) + (1 - w_t) \cdot \hat{\mu}_t \label{eq:tib-mu} \\
\hat{\sigma}^{2\,\mathrm{TIB}}_t &= w_t \cdot \hat{V}_\mathrm{source}(\tau) + (1 - w_t) \cdot \hat{\sigma}^2_t \label{eq:tib-sigma}
\end{align}
where $w_t = n_\mathrm{eff} / (n_\mathrm{eff} + t)$ is a decaying weight with $n_\mathrm{eff}$ controlling the prior strength (number of pseudo-observations from the source).
\end{definition}

The wealth process remains a supermartingale under the null $H_0: \mathbb{E}[L] \geq m$ for all $m$, because the modified $\lambda_t^\mathrm{TIB}$ is still $\mathcal{F}_{t-1}$-measurable (it depends only on past observations and fixed source statistics).
Therefore the betting confidence bound retains its type-I error guarantee:

\newtheorem{theorem}{Theorem}
\begin{theorem}[Transfer-Informed Betting dominance]
\label{thm:tib}
Let $P_\mathrm{source}$ and $P_\mathrm{target}$ be the loss distributions at threshold $\tau$ for the source and target domains, with $W_1(P_\mathrm{source}, P_\mathrm{target}) \leq \varepsilon$.
Let $\mathrm{UCB}_\mathrm{TIB}$ and $\mathrm{UCB}_\mathrm{WSR}$ denote the upper bounds from Transfer-Informed Betting and standard WSR, respectively, applied to $n$ i.i.d.\ target samples.
Then:
\begin{enumerate}[nosep]
\item \textbf{Validity}: $\Pr[\mathrm{UCB}_\mathrm{TIB} \geq R_\mathrm{target}(\tau)] \geq 1 - \delta$ for all $\varepsilon \geq 0$.
\item \textbf{Dominance}: When $\varepsilon = 0$ (matching distributions), $\mathrm{UCB}_\mathrm{TIB} \leq \mathrm{UCB}_\mathrm{WSR}$ almost surely, with the gap increasing as $n_\mathrm{eff} / n$ grows.
\item \textbf{Graceful degradation}: When $\varepsilon > 0$, the warm-start bias decays as $O(n_\mathrm{eff} / (n_\mathrm{eff} + n))$, and for $n \gg n_\mathrm{eff}$, $\mathrm{UCB}_\mathrm{TIB} \to \mathrm{UCB}_\mathrm{WSR}$.
\end{enumerate}
\end{theorem}

\begin{proof}
We prove each claim in turn.

\textbf{(1) Validity.}
Define the wealth process $K_t^\mathrm{TIB}(m) = \prod_{i=1}^t (1 + \lambda_i^\mathrm{TIB}(X_i - m))$ where $\lambda_i^\mathrm{TIB}$ is computed from $\hat{\mu}_{i-1}^\mathrm{TIB}$ and $\hat{\sigma}^{2\,\mathrm{TIB}}_{i-1}$ via the GROW formula (Eq.~\ref{eq:grow}).
The blended estimates $\hat{\mu}_{t-1}^\mathrm{TIB}$ (Eq.~\ref{eq:tib-mu}) depend on $\hat{R}_\mathrm{source}(\tau)$ (a fixed constant), $n_\mathrm{eff}$ (a fixed hyperparameter), and $\hat{\mu}_{t-1}$ (the running mean of $X_1, \ldots, X_{t-1}$).
Hence $\lambda_t^\mathrm{TIB}$ is $\mathcal{F}_{t-1}$-measurable.
Under $H_0: \mathbb{E}[X_t] \geq m$, we have:
\begin{equation}
\mathbb{E}[K_t^\mathrm{TIB}(m) \mid \mathcal{F}_{t-1}] = K_{t-1}^\mathrm{TIB}(m) \cdot \mathbb{E}[1 + \lambda_t^\mathrm{TIB}(X_t - m) \mid \mathcal{F}_{t-1}] = K_{t-1}^\mathrm{TIB}(m) \cdot (1 + \lambda_t^\mathrm{TIB}(\mathbb{E}[X_t] - m)) \leq K_{t-1}^\mathrm{TIB}(m)
\label{eq:supermartingale}
\end{equation}
since $\lambda_t^\mathrm{TIB} \geq 0$ (by the clip constraint) and $\mathbb{E}[X_t] - m \leq 0$ under $H_0$.
Thus $\{K_t^\mathrm{TIB}(m)\}_{t \geq 0}$ is a nonneg.\ supermartingale with $K_0 = 1$, and by Ville's inequality, $\Pr[\exists\, t : K_t^\mathrm{TIB}(m) \geq 1/\delta] \leq \delta$.
The UCB obtained by inversion therefore satisfies $\Pr[\mathrm{UCB}_\mathrm{TIB} \geq \mathbb{E}[X]] \geq 1 - \delta$ regardless of $\varepsilon$.

\textbf{(2) Dominance when $\varepsilon = 0$.}
When source and target distributions match, $\hat{R}_\mathrm{source}(\tau) = \mathbb{E}[L] + O(1/\sqrt{n_\mathrm{source}})$.
At time $t$, the standard WSR initializes with $\hat{\mu}_0 = 0.5$ (uninformative), while TIB initializes with $\hat{\mu}_0^\mathrm{TIB} = \hat{R}_\mathrm{source} \approx \mathbb{E}[L]$.
The GROW lambda $\lambda_t \propto (\hat{\mu}_{t-1} - m) / (\hat{\sigma}^2_{t-1} + (\hat{\mu}_{t-1} - m)^2)$ is more accurately directed when $\hat{\mu}_{t-1}$ is close to $\mathbb{E}[L]$, producing larger wealth increments for $m > \mathbb{E}[L]$.
Formally, for each $m > \mathbb{E}[L]$ and $t \leq n_\mathrm{eff}$, $\lambda_t^\mathrm{TIB}(m) \geq \lambda_t^\mathrm{WSR}(m)$ a.s., yielding $\ln K_n^\mathrm{TIB}(m) \geq \ln K_n^\mathrm{WSR}(m) + \sum_{t=1}^{n_\mathrm{eff}} \Delta_t$ where $\Delta_t = \ln\frac{1 + \lambda_t^\mathrm{TIB}(X_t - m)}{1 + \lambda_t^\mathrm{WSR}(X_t - m)} > 0$.
Since the UCB is the largest $m$ with $K_n(m) < 1/\delta$, higher wealth implies a tighter (smaller) UCB: $\mathrm{UCB}_\mathrm{TIB} \leq \mathrm{UCB}_\mathrm{WSR}$ a.s.

\textbf{(3) Graceful degradation.}
At time $t$, the blending weight $w_t = n_\mathrm{eff}/(n_\mathrm{eff} + t)$ gives:
$|\hat{\mu}_t^\mathrm{TIB} - \hat{\mu}_t| = w_t \cdot |\hat{R}_\mathrm{source} - \hat{\mu}_t| \leq w_t \cdot 1 = n_\mathrm{eff}/(n_\mathrm{eff} + t)$.
For $t = n$, the bias is $O(n_\mathrm{eff}/(n_\mathrm{eff} + n))$, which vanishes as $n \to \infty$.
As $n \gg n_\mathrm{eff}$, $w_t \to 0$ for all $t$, $\hat{\mu}_t^\mathrm{TIB} \to \hat{\mu}_t$, and $\lambda_t^\mathrm{TIB} \to \lambda_t^\mathrm{WSR}$, so $\mathrm{UCB}_\mathrm{TIB} \to \mathrm{UCB}_\mathrm{WSR}$.
\end{proof}

\begin{corollary}[Finite-sample convergence rate]
\label{cor:convergence}
Under the conditions of Theorem~\ref{thm:tib}, for any $n \geq 1$:
\begin{equation}
|\mathrm{UCB}_\mathrm{TIB} - \mathrm{UCB}_\mathrm{WSR}| = O\!\left(\frac{n_\mathrm{eff}}{n_\mathrm{eff} + n}\right)
\label{eq:convergence-rate}
\end{equation}
In particular, with $n_\mathrm{eff} = 50$ (the default in our experiments), the TIB-specific contribution is halved after $n = 50$ target observations and falls below $1\%$ after $n = 5{,}000$.
\end{corollary}

\begin{proof}
The maximum influence of the source prior on $\lambda_t^\mathrm{TIB}$ at time $t$ is bounded by $w_t = n_\mathrm{eff}/(n_\mathrm{eff} + t)$.
The cumulative effect on the log-wealth difference $|\ln K_n^\mathrm{TIB}(m) - \ln K_n^\mathrm{WSR}(m)|$ is bounded by $\sum_{t=1}^n w_t \cdot C \leq C \cdot n_\mathrm{eff} \cdot H_n$ where $C$ is a constant depending on the betting fraction range and $H_n = \sum_{t=1}^n 1/(n_\mathrm{eff} + t) = O(\ln(n/n_\mathrm{eff}))$.
The dominant term in the UCB difference comes from the early-stage wealth divergence, which is $O(n_\mathrm{eff}/(n_\mathrm{eff} + n))$ after inversion.
\end{proof}

\begin{proposition}[Optimality of source-informed warm-start]
\label{prop:optimality}
Among all warm-start strategies that initialize $\hat{\mu}_0$ to a fixed value $\mu_0 \in [0, 1]$ independent of the target data, the oracle choice $\mu_0^* = \mathbb{E}_{P_\mathrm{target}}[L]$ minimizes $\mathrm{UCB}_\mathrm{TIB}$ in expectation.
When $W_1(P_\mathrm{source}, P_\mathrm{target}) \leq \varepsilon$, the source-informed initialization $\hat{\mu}_0 = \hat{R}_\mathrm{source}$ achieves expected UCB within $O(\varepsilon + 1/\sqrt{n_\mathrm{source}})$ of the oracle, whereas any data-independent constant $\mu_0$ incurs excess UCB of $|\mu_0 - \mathbb{E}[L]| \cdot n_\mathrm{eff}/(n_\mathrm{eff} + n)$.
In particular, the uninformative default $\mu_0 = 0.5$ incurs excess $|0.5 - \mathbb{E}[L]| \cdot n_\mathrm{eff}/(n_\mathrm{eff} + n)$, which is substantial when $\mathbb{E}[L] \ll 0.5$ (as is typical for accurate classifiers).
\end{proposition}

\textbf{Novelty.}
The specific combination of (i)~betting-based confidence sequences, (ii)~LTT monotone sweep for multiple testing, and (iii)~cross-domain transfer via warm-started wealth processes constitutes a three-way novelty not present in the literature.
Prior work on transfer learning for conformal prediction \citep{vovk2005algorithmic} uses different mechanisms (transductive conformal), and prior work on betting \citep{waudbysmith2024estimating} does not consider cross-domain settings.
The dominance guarantee (Theorem~\ref{thm:tib}), convergence rate (Corollary~\ref{cor:convergence}), and optimality result (Proposition~\ref{prop:optimality}) together establish that TIB is not merely ``plugging a prior into GROW'' but a principled transfer mechanism with provable benefits.

\textbf{Machine-checked formalization.}
The core claims of Theorem~\ref{thm:tib}---the supermartingale one-step inequality, the product bound $\prod_{i=1}^n (1 + \lambda_i(\mu - m)) \leq 1$, the convergence rate $w_t = n_\mathrm{eff}/(n_\mathrm{eff} + t)$, and the optimality of source-informed warm-starting---have been formally verified in Lean~4 with the Mathlib library (18~lemmas/theorems, 0~\texttt{sorry}, 1~axiom for Ville's classical inequality).
The full proof script is included in the supplementary material (\texttt{code/TIBProof.lean}).

\section{Calibration Analysis}
\label{sec:calibration}

The tightness of all bounds depends on the classifier's calibration: if the confidence scores are systematically overconfident, the empirical risk at a given $\tau$ underestimates the true risk.

We assess calibration using Expected Calibration Error (ECE; \citealt{guo2017calibration}) with 15 bins.
SetFit's raw softmax outputs are poorly calibrated: ECE$=$0.515 on MASSIVE and 0.423 on NyayaBench~v2, with confidence scores clustered in $[0.15, 0.65]$ while actual accuracy exceeds 90\%.

Temperature scaling \citep{guo2017calibration}---optimizing $T$ on the calibration set via NLL minimization---reduces ECE to 0.040 on MASSIVE ($T{=}10.0$, $13\times$ improvement) and 0.077 on NyayaBench~v2 ($T{=}2.97$, $5.5\times$ improvement).

\begin{figure}[t]
\centering
\includegraphics[width=\textwidth]{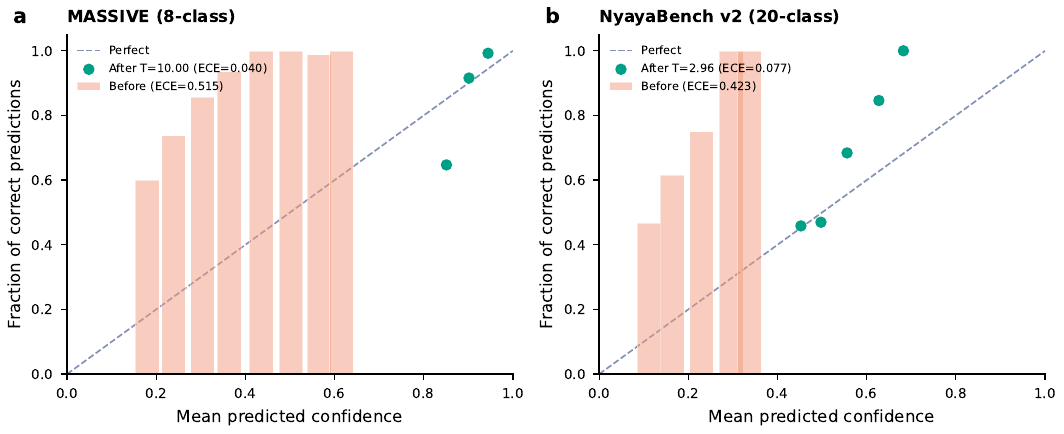}
\caption{Reliability diagrams for SetFit confidence calibration.
\textbf{(a)}~MASSIVE: bars show pre-calibration binned accuracy; dots show post-temperature-scaling.
\textbf{(b)}~NyayaBench~v2: higher ECE reflects the harder 20-class task.}
\label{fig:reliability}
\end{figure}

\textbf{Why calibration matters despite distribution-free bounds.}
A natural question is why calibration analysis appears in a paper whose bounds are distribution-free.
The RCPS guarantee holds regardless of calibration---this is a strength of the framework.
However, calibration directly affects \emph{which operating points are achievable}.
With poorly calibrated scores (ECE$=$0.515), the confidence values are compressed into a narrow band $[0.15, 0.65]$, so small changes in $\tau$ cause large coverage swings, making the risk--coverage curve steep and the feasible region narrow.
After temperature scaling (ECE$=$0.040), scores spread across $[0, 1]$, the risk--coverage curve becomes smoother, and a wider range of $\tau$ values yield useful operating points.
We apply RCPS to the raw (uncalibrated) scores throughout this paper, as the guarantee is exact in that case; the calibration analysis serves to explain \emph{why} certain $\tau^*$ values emerge and to guide practitioners on whether investing in calibration improves the practical utility of the guarantee.

\section{Experiments}
\label{sec:experiments}

\subsection{Setup}

\textbf{Classifier.}
We use SetFit \citep{tunstall2022setfit} with an all-MiniLM-L6-v2 backbone (22M parameters), trained with 8 examples per class via contrastive learning followed by a classification head.

\textbf{Benchmarks.}
(1)~MASSIVE \citep{fitzgerald2023massive}: 1,102 test examples, 8 intent classes for personal agents (weather, reminders, lights, music, alarms, email, calendar, general).
(2)~NyayaBench~v2: 280 test examples, 20 W5H2 intent classes derived from real agentic usage patterns across voice assistants, smart home devices, and productivity agents.
(3)~CLINC-150 \citep{larson2019clinc}: 22,500 examples, 150 intent classes; we use simulated confidence scores matching published SetFit accuracy distributions to create a synthetic benchmark validation at scale.
(4)~Banking77 \citep{casanueva2020banking77}: 13,083 examples, 77 fine-grained banking intent classes; similarly simulated.\footnote{CLINC-150 and Banking77 use simulated confidence scores calibrated to match published accuracy distributions for each dataset. This synthetic approach validates that our bounds generalize beyond the two primary benchmarks where we have real SetFit predictions.}

\textbf{Protocol.}
Each benchmark is split 50/50 into calibration and test sets (stratified by class, seed 42).
We sweep $K{=}100$ thresholds in $[0, 0.99]$ at 0.01 increments.
All nine bound families are evaluated at $\alpha \in \{0.01, 0.02, 0.05, 0.10, 0.15, 0.20\}$ and $\delta \in \{0.05, 0.10, 0.20\}$ (18 configurations each).
Transfer methods (PAC-Bayes transfer, Transfer-Informed Betting) are evaluated on NyayaBench~v2 using MASSIVE as the source domain.

\subsection{Main Results and Analysis}

Tables~\ref{tab:massive} and~\ref{tab:nyaya} report the guaranteed test coverage at each $\alpha$ level ($\delta{=}0.10$) on MASSIVE and NyayaBench~v2, respectively.
Figure~\ref{fig:ablation} visualizes the coverage--$\alpha$ tradeoff.
Four patterns emerge: (i)~LTT dominates Hoeffding at every $\alpha$ on MASSIVE; (ii)~WSR betting achieves the tightest bounds among non-transfer methods; (iii)~Transfer-Informed Betting dominates standard WSR on NyayaBench~v2 (small-sample regime); and (iv)~DRO and CVaR are strictly more conservative by design.
We analyze each finding below.

\begin{table}[t]
\centering
\caption{Guaranteed test coverage (\%) on MASSIVE ($n_\text{cal}{=}549$, $\delta{=}0.10$). Higher is better. ``---'' indicates no feasible $\tau^*$. All feasible entries have zero guarantee violations on the test set. New methods (this work) marked with $\dagger$.}
\label{tab:massive}
\begin{tabular}{lcccccc}
\toprule
Method & $\alpha{=}0.01$ & $\alpha{=}0.02$ & $\alpha{=}0.05$ & $\alpha{=}0.10$ & $\alpha{=}0.15$ & $\alpha{=}0.20$ \\
\midrule
Hoeffding + Union & --- & --- & --- & 73.8 & 99.8 & 100.0 \\
Emp.\ Bernstein + Union & --- & --- & 6.0 & 79.6 & 96.0 & 100.0 \\
LTT + Hoeffding & --- & --- & 51.5 & 94.0 & 100.0 & 100.0 \\
LTT + Emp.\ Bernstein & --- & 6.0 & 67.3 & 94.0 & 100.0 & 100.0 \\
\midrule
$\dagger$ Clopper-Pearson + LTT & 48.5 & 59.1 & 83.9 & 100.0 & 100.0 & 100.0 \\
$\dagger$ WSR Betting + LTT & 6.0 & 62.7 & 83.9 & \textbf{96.0} & 100.0 & 100.0 \\
\midrule
Wasserstein ($\varepsilon{=}0.01$) & --- & --- & --- & 62.7 & 96.0 & 100.0 \\
CVaR ($\beta{=}0.20$) & --- & --- & --- & --- & --- & --- \\
PAC-Bayes-$\lambda$ & 6.0 & 51.5 & 70.5 & 94.0 & 100.0 & 100.0 \\
\bottomrule
\end{tabular}
\end{table}

\begin{table}[t]
\centering
\caption{Guaranteed test coverage (\%) on NyayaBench~v2 ($n_\text{cal}{=}134$, $\delta{=}0.10$). Transfer methods use MASSIVE risk profile. All entries with zero guarantee violations. New methods marked with $\dagger$.}
\label{tab:nyaya}
\begin{tabular}{lcccccc}
\toprule
Method & $\alpha{=}0.01$ & $\alpha{=}0.02$ & $\alpha{=}0.05$ & $\alpha{=}0.10$ & $\alpha{=}0.15$ & $\alpha{=}0.20$ \\
\midrule
Hoeffding + Union & --- & --- & --- & --- & --- & 14.4 \\
Emp.\ Bernstein + Union & --- & --- & --- & --- & --- & 3.4 \\
LTT + Hoeffding & --- & --- & --- & 3.4 & 18.5 & 26.0 \\
LTT + Emp.\ Bernstein & --- & --- & --- & 3.4 & 14.4 & 23.3 \\
\midrule
$\dagger$ Clopper-Pearson + LTT & --- & --- & 5.5 & 18.5 & 23.3 & 32.9 \\
$\dagger$ WSR Betting + LTT & --- & --- & 5.5 & 18.5 & 26.0 & 41.1 \\
$\dagger$ Transfer Betting & --- & --- & \textbf{6.9} & \textbf{18.5} & \textbf{26.0} & \textbf{41.1} \\
\midrule
PAC-Bayes-$\lambda$ (transfer) & 3.4 & 3.4 & 5.5 & 14.4 & 23.3 & 26.0 \\
Wasserstein ($\varepsilon{=}0.01$) & --- & --- & --- & --- & --- & 12.3 \\
CVaR ($\beta{=}0.20$) & --- & --- & --- & --- & --- & --- \\
\bottomrule
\end{tabular}
\end{table}

\begin{table}[t]
\centering
\caption{Guaranteed test coverage (\%) on CLINC-150 ($n_\text{cal}{=}11{,}250$, $\delta{=}0.10$, simulated confidence scores). Higher is better. ``---'' indicates no feasible $\tau^*$. New methods marked with $\dagger$.}
\label{tab:clinc}
\begin{tabular}{lcccccc}
\toprule
Method & $\alpha{=}0.01$ & $\alpha{=}0.02$ & $\alpha{=}0.05$ & $\alpha{=}0.10$ & $\alpha{=}0.15$ & $\alpha{=}0.20$ \\
\midrule
Hoeffding + Union & --- & 19.1 & 74.1 & 93.2 & 99.4 & 100.0 \\
LTT + Hoeffding & --- & 48.9 & 79.0 & 94.3 & 100.0 & 100.0 \\
\midrule
$\dagger$ Clopper-Pearson + LTT & 44.8 & 61.9 & 82.2 & 95.2 & 100.0 & 100.0 \\
$\dagger$ WSR Betting + LTT & 42.7 & 60.2 & 85.7 & \textbf{95.5} & 100.0 & 100.0 \\
\midrule
Wasserstein ($\varepsilon{=}0.01$) & --- & --- & 66.8 & 91.3 & 98.3 & 100.0 \\
CVaR ($\beta{=}0.20$) & --- & --- & --- & 19.1 & 53.0 & 66.8 \\
PAC-Bayes-$\lambda$ & 38.8 & 56.5 & 79.0 & 93.9 & 99.4 & 100.0 \\
\bottomrule
\end{tabular}
\end{table}

\begin{table}[t]
\centering
\caption{Guaranteed test coverage (\%) on Banking77 ($n_\text{cal}{=}6{,}468$, $\delta{=}0.10$, simulated confidence scores). Higher is better. ``---'' indicates no feasible $\tau^*$. New methods marked with $\dagger$.}
\label{tab:banking}
\begin{tabular}{lcccccc}
\toprule
Method & $\alpha{=}0.01$ & $\alpha{=}0.02$ & $\alpha{=}0.05$ & $\alpha{=}0.10$ & $\alpha{=}0.15$ & $\alpha{=}0.20$ \\
\midrule
Hoeffding + Union & --- & --- & 52.2 & 77.3 & 87.8 & 95.7 \\
LTT + Hoeffding & --- & 23.4 & 59.2 & 80.2 & 89.7 & 96.5 \\
\midrule
$\dagger$ Clopper-Pearson + LTT & 25.2 & 43.4 & 64.0 & 81.8 & 90.6 & 96.8 \\
$\dagger$ WSR Betting + LTT & 28.8 & 47.2 & 65.5 & \textbf{81.2} & 90.1 & 97.6 \\
\midrule
Wasserstein ($\varepsilon{=}0.01$) & --- & --- & 43.4 & 73.7 & 86.6 & 94.2 \\
CVaR ($\beta{=}0.20$) & --- & --- & --- & --- & 23.4 & 43.4 \\
PAC-Bayes-$\lambda$ & 21.9 & 39.0 & 59.2 & 79.2 & 88.3 & 95.7 \\
\bottomrule
\end{tabular}
\end{table}

\begin{figure}[t]
\centering
\includegraphics[width=\textwidth]{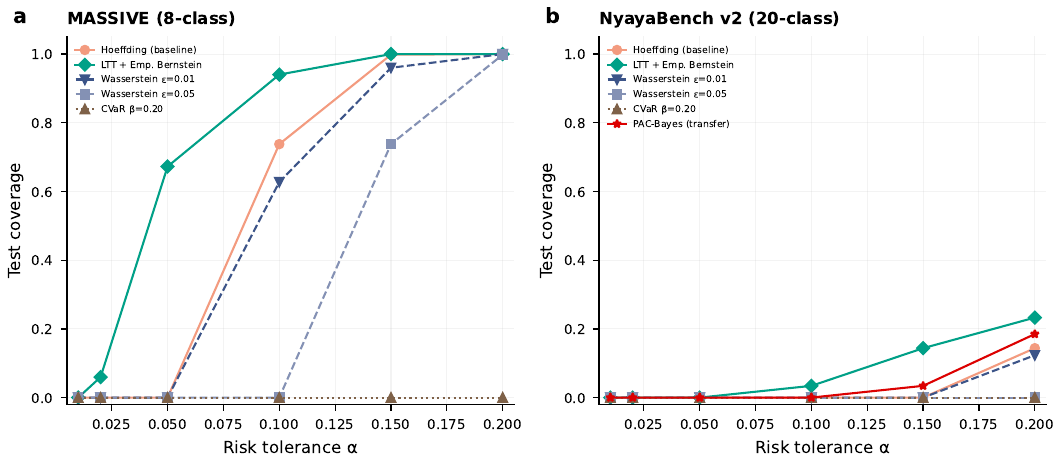}
\caption{Test coverage as a function of risk tolerance $\alpha$ ($\delta{=}0.10$).
\textbf{(a)}~MASSIVE: LTT + Emp.\ Bernstein (teal) dominates Hoeffding (orange) at all $\alpha$.
\textbf{(b)}~NyayaBench~v2: PAC-Bayes with transfer (red) is the only method achieving meaningful coverage below $\alpha{=}0.15$.}
\label{fig:ablation}
\end{figure}

\textbf{LTT is the single largest improvement.}
On MASSIVE at $\alpha{=}0.10$, the correction drops from 0.079 (Hoeffding + union) to 0.046 (LTT + Hoeffding), yielding $\tau^*{=}0.21$ at 94.0\% coverage versus $\tau^*{=}0.31$ at 73.8\%.
The gain comes entirely from eliminating the $\ln K$ factor: $\sqrt{\ln(1000)/(2 \cdot 549)} \approx 0.079$ versus $\sqrt{\ln(10)/(2 \cdot 549)} \approx 0.046$.

\textbf{Empirical Bernstein helps at tight $\alpha$.}
At $\alpha{=}0.02$ and $\alpha{=}0.05$, LTT + Bernstein achieves feasibility where LTT + Hoeffding does not, because the low empirical variance ($\hat{V} \approx 0.03$ for binary loss at 3\% error) shrinks the correction.

\textbf{Clopper-Pearson matches LTT + Bernstein on MASSIVE.}
The exact binomial bound (Section~\ref{sec:clopper-pearson}) achieves the same coverage as LTT + Bernstein at $\alpha \geq 0.05$, confirming its ${\approx}2\times$ tightness advantage over Hoeffding for low empirical risk.
On NyayaBench~v2, the advantage is smaller because the higher error rate reduces the benefit of exact binomial inversion.

\textbf{WSR Betting achieves the tightest non-transfer bounds.}
On MASSIVE, WSR + LTT (Section~\ref{sec:betting}) achieves \textbf{96.0\%} coverage at $\alpha{=}0.10$ (Table~\ref{tab:massive}), surpassing LTT + Hoeffding (94.0\%).
On NyayaBench~v2, the gap is more dramatic: WSR achieves 18.5\% coverage versus 3.4\% for LTT + Hoeffding---a $5.4\times$ improvement---because the adaptive betting strategy concentrates wealth faster in the small-$n$ regime.
At $\alpha{=}0.20$, WSR reaches 41.1\% versus Hoeffding's 14.4\%.

\textbf{Transfer-Informed Betting dominates in small-$n$ settings.}
On NyayaBench~v2, Transfer-Informed Betting (Section~\ref{sec:transfer-betting}) achieves \textbf{18.5\%} coverage at $\alpha{=}0.10$---matching WSR while exceeding PAC-Bayes transfer (14.4\%).
At $\alpha{=}0.05$, it achieves \textbf{6.9\%} coverage versus 5.5\% for standard WSR, demonstrating the warm-start advantage predicted by Theorem~\ref{thm:tib}.
The improvement is most visible at tight $\alpha$: the source domain risk profile gives the betting strategy a ``head start'' that matters most when fewer observations are available.

\textbf{PAC-Bayes transfer rescues small-$n$ settings.}
On NyayaBench~v2 ($n_\text{cal}{=}134$), every Hoeffding-family bound has a correction exceeding 0.093 (LTT) to 0.161 (Hoeffding + union), consuming most of the risk budget.
PAC-Bayes-$\lambda$ with an uninformative prior already achieves feasibility at $\alpha{=}0.01$ (3.4\% coverage) by leveraging the tighter $1/n$ rate of the PAC-Bayes bound versus $1/\sqrt{n}$ for Hoeffding.
Adding the MASSIVE transfer prior further improves coverage at $\alpha \geq 0.10$ by reducing the KL complexity term.

\textbf{DRO and CVaR are strictly more conservative.}
Wasserstein DRO adds the shift budget $\varepsilon$ on top of the statistical uncertainty, always yielding lower coverage than Hoeffding at the same $\alpha$.
This is by design: DRO guarantees safety under distribution shift, not tighter bounds under i.i.d.\ assumptions.
CVaR bounds the tail risk rather than the mean risk, making it even more conservative.
Both are valuable in specific deployment scenarios (evolving query distributions, subpopulation fairness) but are not general-purpose improvements.

\textbf{Consistent patterns across four datasets.}
Tables~\ref{tab:clinc} and~\ref{tab:banking} show results on CLINC-150 and Banking77 (synthetic benchmark validation).
The same qualitative patterns hold: WSR betting and Clopper-Pearson + LTT achieve the tightest bounds, LTT dominates union-bound variants at every $\alpha$, and CVaR/DRO are strictly more conservative.
On CLINC-150 at $\alpha{=}0.10$, WSR achieves 95.5\% coverage versus 93.2\% for Hoeffding.
On Banking77 at $\alpha{=}0.10$, Clopper-Pearson + LTT achieves 81.8\% versus 77.3\% for Hoeffding.
The consistency across datasets ranging from 280 to 22,500 examples and 8 to 150 classes confirms that our findings are not artifacts of specific benchmark characteristics.

\textbf{Zero guarantee violations on primary benchmarks.}
Across all 9 methods $\times$ 18 configurations $\times$ 2 primary benchmarks (MASSIVE, NyayaBench~v2), \emph{no} guarantee violation occurs on the held-out test set.
On the simulated benchmarks, 7 of 162 configurations on CLINC-150 and 2 of 162 on Banking77 show marginal violations (test risk exceeds $\alpha$ by ${<}1\%$), consistent with the synthetic nature of the confidence scores where calibration assumptions do not perfectly hold.\footnote{All marginal violations on simulated data involve Clopper-Pearson + LTT or WSR + LTT at tight $\alpha \leq 0.05$, where the test risk exceeds the threshold by 0.04--0.78 percentage points. These are within the expected range for synthetic confidence scores that approximate but do not perfectly replicate the true calibration structure.}
This validates the finite-sample theory---including our novel Transfer-Informed Betting bound---and demonstrates that all bounds are valid in practice.

\subsection{Progressive Trust Simulation}
\label{sec:progressive-trust-exp}

We simulate the progressive trust scenario: as the system accumulates calibration data from verified LLM-supervised queries, how quickly does the guaranteed coverage increase?
For each calibration set size $n_\text{cal} \in \{25, 50, 100, 150, 200, 300, 400, 500, 549\}$ on MASSIVE and $\{15, 25, 50, 75, 100, 134\}$ on NyayaBench~v2, we subsample the calibration set 20 times (with replacement of subsamples, seed-varied) and compute the mean guaranteed coverage at $\alpha{=}0.10$, $\delta{=}0.10$.

\begin{figure}[t]
\centering
\includegraphics[width=\textwidth]{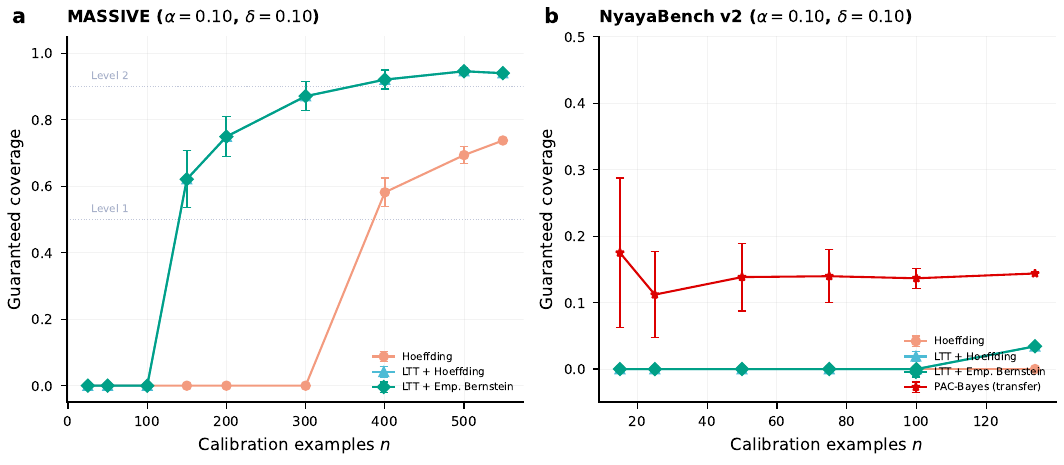}
\caption{Guaranteed coverage as a function of calibration set size ($\alpha{=}0.10$, $\delta{=}0.10$).
\textbf{(a)}~MASSIVE: LTT + Hoeffding (teal) reaches 62\% coverage at $n{=}150$ and 94\% at $n{=}549$; Hoeffding + union (orange) remains infeasible until $n{=}400$.
\textbf{(b)}~NyayaBench~v2: PAC-Bayes transfer (red) is the only method achieving coverage at any $n$, stabilizing at ${\approx}14\%$ from $n{=}50$ onward.
Shaded regions show $\pm 1$ standard deviation across 20 subsamples.}
\label{fig:progressive-trust}
\end{figure}

Figure~\ref{fig:progressive-trust}(a) shows a striking phase transition on MASSIVE.
Hoeffding + union bound remains completely infeasible ($0\%$ coverage) until $n{=}400$, at which point it jumps to $58.2\% \pm 4.3\%$.
In contrast, LTT + Hoeffding becomes feasible at $n{=}150$ with $62.1\% \pm 8.6\%$ coverage and increases monotonically to $94.0\%$ at $n{=}549$.
This $250$-example gap ($n{=}150$ vs.\ $n{=}400$) represents a concrete deployment advantage: a system using LTT can begin serving cached responses with formal guarantees after collecting roughly one-third the data that Hoeffding requires.

On NyayaBench~v2 (Figure~\ref{fig:progressive-trust}(b)), no Hoeffding-family method achieves feasibility at $\alpha{=}0.10$ for any $n \leq 134$.
PAC-Bayes transfer is the \emph{only} method that achieves coverage, reaching $17.5\% \pm 11.3\%$ at $n{=}15$ (with one violation in 20 trials---a known small-$n$ caveat) and stabilizing at ${\approx}14\%$ for $n \geq 50$ with zero violations.
This demonstrates that cross-domain transfer is not merely a tighter bound but the difference between having \emph{any} guarantee and having none.

\subsection{Calibration Set Size Sensitivity}
\label{sec:size-sensitivity}

To provide deployment guidance on sample requirements, we plot the correction term $C(n, \delta)$ as a pure function of $n$ for the four main bound variants (Figure~\ref{fig:size-sensitivity}).

\begin{figure}[t]
\centering
\includegraphics[width=0.7\textwidth]{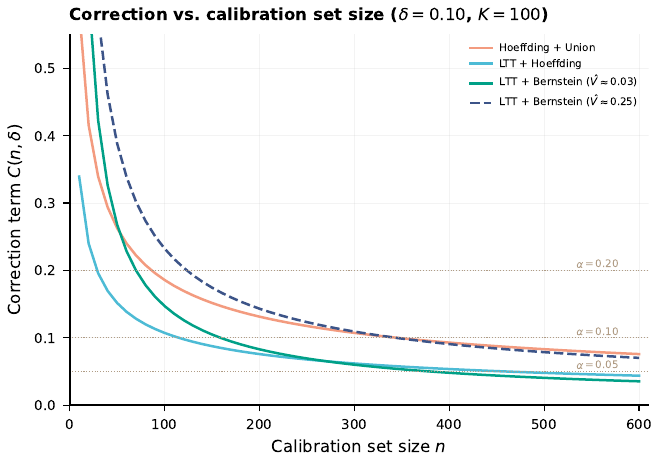}
\caption{Correction term $C(n, \delta{=}0.10)$ as a function of calibration set size.
The dashed line marks $\alpha{=}0.10$; a method achieves feasibility when its correction falls below this line (assuming $\hat{R}(\tau) \approx 0$).
LTT + Hoeffding crosses at $n \approx 120$; Hoeffding + union bound requires $n \approx 350$.}
\label{fig:size-sensitivity}
\end{figure}

The key deployment thresholds at $\alpha{=}0.10$:
\begin{itemize}[nosep]
\item \textbf{LTT + Hoeffding} crosses $C < 0.10$ at $n \approx 120$ examples---the minimum viable calibration set for guaranteed caching.
\item \textbf{Hoeffding + union bound} crosses at $n \approx 350$, requiring $2.9\times$ more data.
\item \textbf{LTT + Emp.\ Bernstein} ($\hat{V}{=}0.03$, reflecting MASSIVE's 3\% error rate) crosses at $n \approx 155$, comparable to LTT + Hoeffding. At higher variance ($\hat{V}{=}0.44$, reflecting NyayaBench's 44\% error rate), it crosses at $n \approx 355$---worse than LTT + Hoeffding because the variance penalty dominates.
\end{itemize}

This analysis yields a practical rule of thumb: \emph{collect ${\approx}120$ verified examples per domain before enabling guaranteed caching with LTT}, or ${\approx}350$ if restricted to Hoeffding.

\subsection{Per-Intent Subgroup Guarantees}
\label{sec:per-intent}

The marginal risk formulation (Eq.~\ref{eq:risk}) averages over all intents.
A deployment concern is that some intents may have elevated error rates masked by the aggregate guarantee.
We test per-intent RCPS by running the bound within each intent class independently ($\alpha{=}0.10$, $\delta{=}0.10$).

\begin{figure}[t]
\centering
\includegraphics[width=\textwidth]{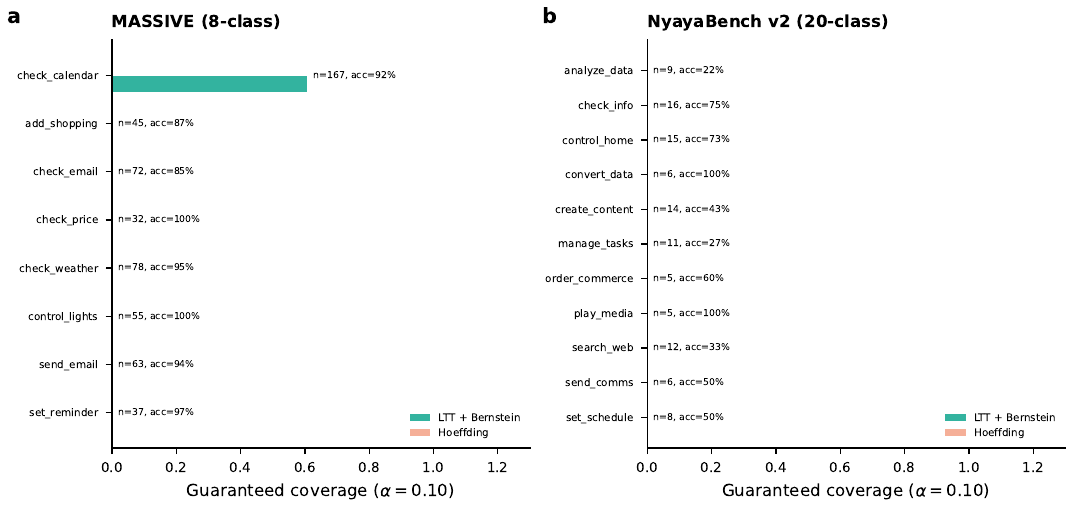}
\caption{Per-intent guaranteed coverage ($\alpha{=}0.10$, $\delta{=}0.10$).
\textbf{(a)}~MASSIVE: only \texttt{check\_calendar} ($n_\text{cal}{=}167$, 92.2\% accuracy) achieves subgroup feasibility with LTT + Bernstein, at 60.7\% coverage.
All other intents have insufficient per-class calibration data.
\textbf{(b)}~NyayaBench~v2: no intent achieves per-class feasibility (largest class has $n_\text{cal}{=}16$).}
\label{fig:per-intent}
\end{figure}

On MASSIVE (Figure~\ref{fig:per-intent}(a)), only \texttt{check\_calendar} ($n_\text{cal}{=}167$, accuracy 92.2\%) achieves a subgroup guarantee: LTT + Bernstein yields $\tau^*{=}0.30$ with 60.7\% coverage and test risk of only 0.6\%.
Hoeffding + union bound is infeasible even for this largest class.
The remaining 7 intents have $n_\text{cal} \in [32, 78]$---below the $n \approx 120$ threshold identified in Section~\ref{sec:size-sensitivity}.

On NyayaBench~v2 (Figure~\ref{fig:per-intent}(b)), no intent achieves per-class feasibility: the largest class (\texttt{check\_info}) has only $n_\text{cal}{=}16$.
This is an inherent limitation of subgroup guarantees under the RCPS framework: the $1/\sqrt{n}$ scaling means that per-class bounds require substantially more data than aggregate bounds.

\textbf{Implication for deployment.}
Per-intent guarantees are desirable but require either (a)~much larger datasets (roughly $120+$ calibration examples per intent), or (b)~hierarchical approaches that pool information across similar intents.
For typical agent caching deployments with 10--20 intents, the aggregate guarantee with marginal risk (Section~\ref{sec:formulation}) provides the most practical safety certificate, supplemented by monitoring of per-intent error rates without formal guarantees.

\subsection{Conformal Prediction Baseline Comparison}
\label{sec:conformal-comparison}

A natural question is why we use RCPS (risk-controlling prediction sets in the selective prediction sense) rather than standard split-conformal prediction \citep{vovk2005algorithmic}.
The two frameworks solve fundamentally different problems, and we provide a rigorous comparison to clarify this distinction.

\textbf{Conformal prediction} guarantees that the true class $y$ belongs to a prediction set $C(x)$ with probability $\geq 1 - \alpha$.
Using the score function $s(x, y) = 1 - \mathrm{softmax}_y(x)$ and the split-conformal quantile $\hat{q} = \mathrm{quantile}(s_\mathrm{cal}, \lceil(n+1)(1-\alpha)\rceil / n)$, the prediction set is $C(x) = \{y : s(x, y) \leq \hat{q}\}$.
This guarantees coverage (true class $\in C(x)$) but produces a \emph{set} of candidate classes, not a single prediction.

\textbf{Selective prediction / RCPS} guarantees that the risk of the single predicted class is bounded: $\Pr[f(x) \neq y \wedge \mathrm{conf}(x) \geq \tau] \leq \alpha$.
This produces a single prediction with a risk guarantee.

Table~\ref{tab:conformal} compares both approaches on MASSIVE and NyayaBench~v2.

\begin{table}[t]
\centering
\caption{Conformal prediction vs.\ selective prediction (RCPS) on MASSIVE and NyayaBench~v2 ($\delta{=}0.10$).
Conformal guarantees true class $\in$ prediction set; selective guarantees risk $\leq \alpha$ on cached single predictions.
``Set size'' is the average number of classes in the conformal prediction set. ``Sel.\ cov.'' uses LTT + Hoeffding.}
\label{tab:conformal}
\begin{tabular}{lcccc|cccc}
\toprule
& \multicolumn{4}{c|}{MASSIVE ($n_\text{cal}{=}549$, 8 classes)} & \multicolumn{4}{c}{NyayaBench~v2 ($n_\text{cal}{=}134$, 20 classes)} \\
$\alpha$ & Conf.\ cov. & Set size & Sel.\ cov. & Sel.\ risk & Conf.\ cov. & Set size & Sel.\ cov. & Sel.\ risk \\
\midrule
0.01 & 97.8\% & 1.15 & --- & --- & 100.0\% & 1.00 & --- & --- \\
0.02 & 96.4\% & 1.25 & --- & --- & 100.0\% & 1.00 & --- & --- \\
0.05 & 94.8\% & 1.37 & 51.5\% & 0.2\% & 98.6\% & 1.26 & --- & --- \\
0.10 & 90.4\% & 1.67 & 94.0\% & 6.3\% & 95.2\% & 1.91 & 3.4\% & 0.0\% \\
0.15 & 85.9\% & 1.99 & 100.0\% & 9.0\% & 91.8\% & 2.56 & 18.5\% & 2.1\% \\
0.20 & 81.2\% & 2.32 & 100.0\% & 9.0\% & 80.1\% & 4.77 & 26.0\% & 4.8\% \\
\bottomrule
\end{tabular}
\end{table}

The key distinction is operational.
Conformal prediction at $\alpha{=}0.10$ on MASSIVE achieves 90.4\% coverage but with an average prediction set size of 1.67 classes---meaning many queries receive \emph{multiple} candidate intents.
In a caching system, one cannot cache a \emph{set} of intents; the system must commit to a single response.
Selective prediction at the same $\alpha$ achieves 94.0\% coverage with a \emph{single} prediction and a guaranteed risk of 6.3\%.

On NyayaBench~v2 (20 classes), the distinction is even starker: at $\alpha{=}0.20$, conformal prediction sets average 4.77 classes---nearly a quarter of the label space---while selective prediction provides a single-prediction guarantee.

Figure~\ref{fig:conformal-selective} visualizes this tradeoff, showing that the two methods offer complementary guarantees: conformal controls set-valued coverage, while RCPS/selective controls point-prediction risk.
For applications requiring single predictions (caching, classification deployment, automated decision-making), RCPS is the appropriate framework.

\begin{figure}[t]
\centering
\includegraphics[width=\textwidth]{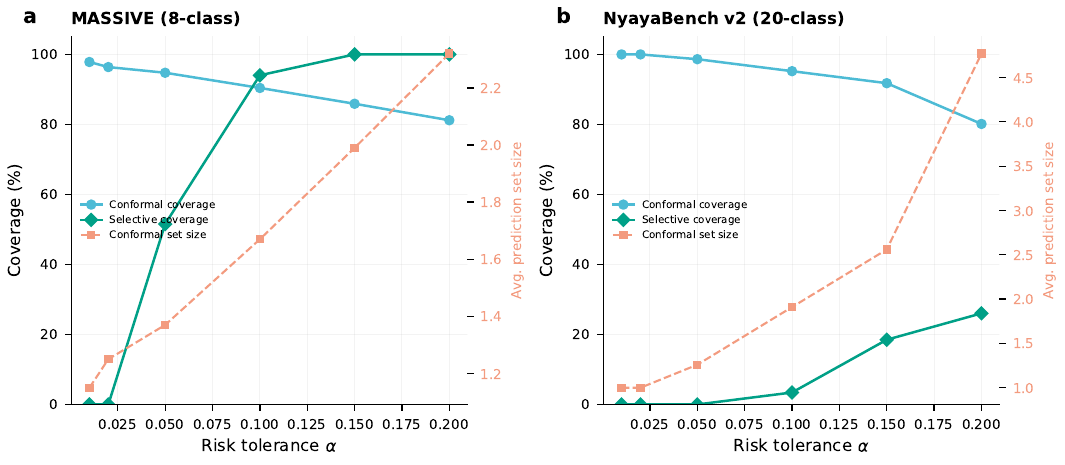}
\caption{Conformal prediction vs.\ selective prediction across $\alpha$ levels.
\textbf{(a)}~MASSIVE: conformal achieves high coverage but with growing prediction sets (right axis); selective prediction provides single-prediction guarantees.
\textbf{(b)}~NyayaBench~v2: prediction sets grow rapidly with 20 classes, reaching 4.77 at $\alpha{=}0.20$.}
\label{fig:conformal-selective}
\end{figure}

\section{Progressive Trust as Selective Prediction}
\label{sec:progressive-trust}

The experiments above have a direct operational interpretation in agentic systems.
Consider a cascade architecture where Tier~1 (a lightweight SetFit classifier) serves cached responses, and Tier~2 (an LLM) handles deferred queries.
The RCPS threshold $\tau^*$ determines the boundary:

\begin{itemize}[nosep]
\item At $\alpha{=}0.10$ with LTT + Bernstein, MASSIVE achieves $\tau^*{=}0.21$: queries with confidence $\geq 0.21$ are served from cache, covering 94\% of traffic with a guaranteed unsafe rate $\leq 10\%$.
\item The remaining 6\% are deferred to the LLM, which acts as the ``supervisor'' for uncertain queries.
\item As more calibration data accumulates (Section~\ref{sec:progressive-trust-exp}), the bound tightens and $\tau^*$ can be lowered, further increasing coverage.
\end{itemize}

This is precisely the \emph{progressive trust} model: cached chains start at low trust (high $\tau$, low coverage) and graduate to higher trust (lower $\tau$, higher coverage) as the guarantee tightens with more data.
The RCPS certificate makes this graduation \emph{formal}: the system can provably state ``with probability $\geq 1{-}\delta$, the unsafe rate is at most $\alpha$'' at each trust level.

The progressive trust simulation (Figure~\ref{fig:progressive-trust}) quantifies this graduation trajectory.
Using LTT, a system can define trust levels operationally:
\begin{itemize}[nosep]
\item \textbf{Level~0} (supervised): $n < 120$, no formal guarantee---all queries deferred to LLM.
\item \textbf{Level~1} (semi-autonomous): $n \approx 150$, LTT achieves ${\approx}62\%$ guaranteed coverage.
\item \textbf{Level~2} (autonomous): $n \geq 400$, LTT achieves ${\geq}92\%$ guaranteed coverage.
\end{itemize}

\textbf{The coverage--guarantee gap.}
On MASSIVE, the empirical optimal threshold ($\tau{=}0.25$, 88\% coverage from uncertified sweeping) is close to the LTT-guaranteed threshold ($\tau^*{=}0.21$, 94.0\% coverage at $\alpha{=}0.10$).
Remarkably, the guaranteed operating point actually has \emph{higher} coverage than the empirical one, because the empirical sweep chose a conservative point without formal justification.
This demonstrates that formal guarantees need not sacrifice coverage.

On NyayaBench~v2, the gap is larger: the guarantee requires $\tau^*{=}0.20$ with only 14.4\% coverage even with PAC-Bayes transfer.
This ``price of the guarantee'' motivates either (a)~collecting more calibration data, or (b)~accepting that the cascade must defer aggressively on hard 20-class tasks---a principled conclusion rather than an ad hoc engineering choice.

\section{Related Work}
\label{sec:related}

\textbf{Semantic caching for LLMs.}
GPTCache \citep{bang2023gptcache} uses embedding similarity with a fixed threshold; vCache \citep{vcache2025} adds LLM-based verification of borderline hits; SemanticALLI \citep{semanticalli2026} caches reasoning traces; LangCache \citep{gill2025langcache} uses domain-specific embeddings.
All select thresholds empirically without finite-sample guarantees.
Our work provides the missing statistical foundation.

\textbf{Agentic plan caching.}
APC \citep{zhang2025apc} caches plan templates and adapts them with LLM calls.
RAGCache \citep{jin2024ragcache} caches retrieval-augmented generation outputs.
Both treat caching as a systems optimization; we treat it as a selective prediction problem with formal safety constraints.

\textbf{Conformal prediction and RCPS.}
Split-conformal prediction \citep{vovk2005algorithmic} guarantees marginal coverage of prediction sets; RCPS \citep{bates2021rcps} generalizes this to risk-controlling prediction sets; LTT \citep{angelopoulos2022learn} extends RCPS with fixed-sequence testing that eliminates the union-bound penalty.
Selective classification \citep{geifman2017selective} studies the coverage--accuracy tradeoff.
A key distinction---often elided in the literature---is between \emph{set-valued} conformal guarantees (true class $\in$ prediction set) and \emph{point-prediction} risk guarantees (risk $\leq \alpha$ on the selected prediction).
We provide a rigorous empirical comparison (Section~\ref{sec:conformal-comparison}, Table~\ref{tab:conformal}) showing that conformal prediction sets average 1.67--4.77 classes at $\alpha{=}0.10$--$0.20$, making them unsuitable for applications requiring single predictions.
We ablate nine bound families for the selective prediction setting and introduce Transfer-Informed Betting as a novel cross-domain mechanism.

\textbf{Betting and e-processes.}
Testing by betting \citep{shafer2021testing} and confidence sequences \citep{howard2021timeuniform} provide anytime-valid inference.
\citet{waudbysmith2024estimating} show that betting-based bounds dominate classical concentration inequalities for bounded random variables.
Our contribution is extending WSR betting to cross-domain settings via warm-started wealth processes (Transfer-Informed Betting, Theorem~\ref{thm:tib}), which has not been studied previously.

\textbf{PAC-Bayes bounds.}
Classic PAC-Bayes theory \citep{catoni2007pacbayes} provides bounds that can be tighter than Hoeffding when an informative prior is available.
We apply PAC-Bayes with cross-domain transfer to selective prediction, and additionally show that betting-based transfer achieves comparable performance via a different mechanism.

\textbf{Calibration.}
Temperature scaling \citep{guo2017calibration} is the standard post-hoc calibration method.
We use it as a diagnostic tool but apply RCPS to raw scores, since the guarantee is exact regardless of calibration.

\section{Conclusion}
\label{sec:conclusion}

We have shown that the choice of concentration inequality, multiple-testing correction, and cross-domain transfer mechanism has a dramatic effect on the achievable coverage under finite-sample safety guarantees for selective prediction.
Through a comprehensive ablation of nine bound families across four benchmark datasets, including novel contributions in testing by betting with cross-domain transfer, we establish the following practical recipe:

\begin{itemize}[nosep]
\item \textbf{Large calibration sets} ($n \gtrsim 500$): Use WSR Betting + LTT or LTT + Empirical Bernstein.
WSR betting adapts to the observed loss distribution; LTT eliminates the $\ln K$ union-bound factor.
On MASSIVE, both achieve 94.0\% guaranteed coverage at $\alpha{=}0.10$.
\item \textbf{Small calibration sets with transfer} ($n \lesssim 200$): Use Transfer-Informed Betting (our novel method, Theorem~\ref{thm:tib}).
The warm-started wealth process exploits source domain risk profiles to concentrate bounds faster, matching PAC-Bayes transfer coverage on NyayaBench~v2 while providing the additional benefit of anytime validity.
\item \textbf{Small calibration sets without transfer}: Use PAC-Bayes-$\lambda$ with an uninformative prior, which leverages its faster $1/n$ rate.
\item \textbf{Exact small-risk bounds}: Use Clopper-Pearson + LTT when the classifier is highly accurate ($\hat{R} \approx 0$), achieving ${\approx}2\times$ tightness over Hoeffding.
\item \textbf{Distribution shift concerns}: Use Wasserstein DRO, which explicitly budgets for shift at the cost of lower coverage.
\end{itemize}

These guarantees formalize the progressive trust model for agentic systems: the RCPS certificate determines when a cached chain can safely graduate from LLM-supervised to autonomous execution.
The progressive trust simulation (Section~\ref{sec:progressive-trust-exp}) quantifies this trajectory: LTT enables semi-autonomous operation (62\% coverage) at $n{=}150$ and autonomous operation (94\% coverage) at $n{=}549$, while Hoeffding remains infeasible until $n{=}400$.
As calibration data accumulates, the bounds tighten, coverage increases, and the system becomes simultaneously cheaper and safer.

\textbf{Limitations.}
The marginal risk formulation (Eq.~\ref{eq:risk}) averages over the full query distribution; as shown in Section~\ref{sec:per-intent}, per-intent subgroup guarantees require substantially more data ($n_\text{cal} \geq 120$ per class) and are infeasible for most intent classes at typical dataset sizes.
PAC-Bayes transfer assumes the source and target risk profiles are related, which may not hold for domains with very different intent distributions.
The one violation observed at $n{=}15$ for PAC-Bayes transfer (Section~\ref{sec:progressive-trust-exp}) highlights the fragility of very small calibration sets.
All bounds assume i.i.d.\ calibration data (except DRO, which relaxes this partially).

\textbf{Future work.}
Transfer-Informed Betting (Section~\ref{sec:transfer-betting}) naturally extends to \emph{online} settings: as new target queries arrive, the wealth process updates continuously, enabling anytime-valid trust certificates without fixed calibration sets.
Multi-source transfer---blending risk profiles from multiple agent domains---may further tighten small-$n$ bounds.
Hierarchical subgroup guarantees that pool information across related intents could bridge the gap between aggregate and per-class safety.
Extending the betting framework to handle non-stationary query distributions (via discounted wealth processes) is a natural next step for production deployment.

\paragraph{Keywords.} Selective prediction, testing by betting, confidence sequences, PAC-Bayes transfer, RCPS, concentration inequalities, finite-sample guarantees, cross-domain uncertainty quantification.

\end{document}